\documentclass[onefignum,onetabnum]{siamart171218}

\usepackage{dsfont}
\usepackage{amsmath,bm}
\usepackage{amsfonts}
\newtheorem{prop}{Proposition}
\newtheorem{rmk}{Remark}
\newtheorem{assumption}{Assumption}

\newcommand{\W}{{\bm W}}

\newcommand{\w}{{\bm w}}
\newcommand{\x}{{\bm x}}

\usepackage{multirow}
\usepackage{booktabs,siunitx}

\usepackage{tikz}
\usepackage{pgfplots}
\pgfplotsset{compat=1.15}
\usetikzlibrary{calc,intersections}

\usepackage{lipsum}
\usepackage{amsfonts}
\usepackage{graphicx}
\usepackage{epstopdf}
\usepackage{algorithmic}
\ifpdf
  \DeclareGraphicsExtensions{.eps,.pdf,.png,.jpg}
\else
  \DeclareGraphicsExtensions{.eps}
\fi

\newsiamremark{remark}{Remark}
\newsiamremark{hypothesis}{Hypothesis}
\crefname{hypothesis}{Hypothesis}{Hypotheses}
\newsiamthm{claim}{Claim}

\headers{Learning Quantized Neural Nets by Coarse Gradient Method}{Ziang Long, Penghang Yin, and Jack Xin}

\title{Learning Quantized Neural Nets by Coarse Gradient Method for Non-linear Classification\thanks{Submitted to the editors DATE.
\funding{This work was funded by NSF grants IIS-1632935, DMS-1854434, DMS-1924548, and DMS-1924935.}}}

\author{Ziang Long\thanks{Department of Mathematics, University of California, Irvine, Irvine, CA 92697 
  (\email{zlong6@uci.edu}).}
\and Penghang Yin\thanks{Department of Mathematics and Statistics, University at Albany, State University of New York, Albany, NY 12222 
  (\email{pyin@albany.edu}).}
\and Jack Xin\thanks{Department of Mathematics, University of California, Irvine, Irvine, CA 92697 
  (\email{jack.xin@uci.edu}).}
}

\usepackage{amsopn}

\makeatletter
\newcommand*{\addFileDependency}[1]{% argument=file name and extension
  \typeout{(#1)}% latexmk will find this if $recorder=0 (however, in that case, it will ignore #1 if it is a .aux or .pdf file etc and it exists! if it doesn't exist, it will appear in the list of dependents regardless)
  \@addtofilelist{#1}% if you want it to appear in \listfiles, not really necessary and latexmk doesn't use this
  \IfFileExists{#1}{}{\typeout{No file #1.}}% latexmk will find this message if #1 doesn't exist (yet)
}
\makeatother

\ifpdf
\hypersetup{
  pdftitle={Learning Quantized Neural Nets by Coarse Gradient Method for Non-linear Classification},
  pdfauthor={Ziang Long, Penghang Yin, and Jack Xin}
}
\fi

\begin{document}

\maketitle

% REQUIRED
\begin{abstract}
Quantized or low-bit neural networks are attractive due to their inference efficiency. However, training deep neural networks with quantized activations involves minimizing a discontinuous and piecewise constant loss function. Such a loss function has zero gradient almost everywhere (a.e.), which makes the conventional gradient-based algorithms inapplicable. 
To this end, we study a novel class of \emph{biased} first-order oracle, termed coarse gradient, for overcoming the vanished gradient issue. A coarse gradient is generated by 
replacing the a.e. zero derivative of quantized (i.e., stair-case) ReLU activation composited in the chain rule with some heuristic proxy derivative called straight-through estimator (STE). Although having been widely used in training quantized networks empirically, fundamental questions like when and why the ad-hoc STE trick works, still lacks theoretical understanding.
In this paper, we propose a class of STEs with certain monotonicity, and consider their applications to the training of a two-linear-layer network with quantized activation functions for non-linear multi-category classification.
We establish performance guarantees for the proposed STEs by showing that the corresponding coarse gradient methods converge to the global minimum, which leads to a perfect classification. Lastly, we present experimental results on synthetic data as well as MNIST dataset to verify our theoretical findings and demonstrate the effectiveness of our proposed STEs.
\end{abstract}

% REQUIRED

% REQUIRED

\section{Introduction}
Deep neural networks (DNNs) have been the main driving force for the recent wave in artificial intelligence (AI). They have achieved remarkable success in a number of domains including computer vision \cite{imagnet_12,faster_rcnn}, reinforcement learning \cite{mnih2015human,silver2016mastering} and natural language processing \cite{collobert2008unified}, to name a few. However, due to the huge number of model parameters, the deployment of DNNs can be computationally and memory intensive. As such, it remains a great challenge to deploy DNNs on mobile electronics with low computational budget and limited memory storage.

Recent efforts have been made to the quantization of weights and activations of DNNs while in the hope of maintaining the accuracy. More specifically, quantization techniques constrain the weights or/and activation values to low-precision arithmetic (e.g. 4-bit) instead of using the conventional floating-point (32-bit) representation \cite{Hubara2017QuantizedNN,dorefa_16,halfwave_17,inq_17,louizos2018relaxed,ttq_16}. In this way, the inference of quantized DNNs translates to hardware-friendly low-bit computations rather than floating-point operations. That being said, quantization brings three critical benefits for AI systems: energy efficiency, memory savings, and inference acceleration.

The approximation power of weight quantized DNNs was investigated in \cite{he2018relu,ding2018universal}, while the recent paper \cite{shen2020deep}  studies the  approximation power  of DNNs with discretized activations. 
On the computational side, training quantized DNNs typically calls for solving a large-scale optimization problem, yet with extra computational and mathematical challenges. 
Although people often quantize both the weights and activations of DNNs, they can be viewed as two relatively independent subproblems. Weight quantization basically introduces an additional set-constraint that characterizes the quantized model parameters, which can be efficiently carried out by projected gradient type methods \cite{courbariaux2015binaryconnect,twn_16,li2017training,yin2016quantization,hou2018loss,yin2018binaryrelax}. Activation quantization (i.e., quantizing ReLU), on the other hand, involves a stair-case activation function with zero derivative almost everywhere (a.e.) in place of the sub-differentiable ReLU. Therefore, the resulting composite loss function is piece-wise constant and cannot be minimized via the (stochastic) gradient method due to the vanished gradient. 

To overcome this issue, a simple and hardware friendly approach is to use a straight-through estimator (STE) \cite{hinton2012neural,bengio2013estimating,yin2018understanding}. More precisely, one can replace the a.e. zero derivative of quantized ReLU with an ad-hoc surrogate in the backward pass, while keeping the original quantized function during the forward pass. 
Mathematically, STE gives rise to a \emph{biased} first-order oracle computed by an unusual chain rule. This first-order oracle is not the gradient of the original loss function because there exists a mismatch between the forward and backward passes. Throughout this paper, this STE-induced type of ``gradient" is called coarse gradient. 
While coarse gradient is not the true gradient, in practice it works as it miraculously points towards a descent
direction (see \cite{yin2018understanding} for a thorough study in the regression setting). Moreover, coarse gradient has the same computational complexity as standard gradient. Just like the standard gradient descent, the minimization procedure of training activation quantized networks simply proceeds by repeatedly moving one step at current point in the opposite of coarse gradient with some step size. The performance of the resulting coarse gradient method, e.g. convergence property, naturally relies on the choice of STE. How to choose a proper STE so that the resulting training algorithm is provably convergent is still poorly understood,  especially in the nonlinear classification setting.

\subsection{Related Works}
% Quantizing activations of DNNs requires to minimize a discrete-valued function, which is similar to the classical perceptron .

The idea of STE dated back to the classical perceptron algorithm \cite{rosenblatt1957perceptron,rosenblatt1962principles} for binary classification. Specifically, the perceptron algorithm attempts to solve the empirical risk minimization problem: 
\begin{equation}\label{eq:model}
  \min_{\w} \; \sum_{i=1}^N (\mbox{sign}(\x_i^{\top}\w) - y_i)^2,  
\end{equation}
where $(\x_i, y_i)$ is the $i^{\mathrm{th}}$ training sample with $y_i\in\{\pm 1\}$ being a binary label; for a given input $\x_i$, the single-layer perceptron model with weights $\w$ outputs the class prediction $\mbox{sign}(\x_i^{\top}\w)$. To train perceptrons, Rosenblatt \cite{rosenblatt1957perceptron} proposed the following iteration for solving (\ref{eq:model}) with the step size $\eta>0$: 
\begin{equation}\label{eq:perceptron}
    \w^{t+1} = \w^{t} - \eta \sum_{i=1}^N (\mbox{sign}(\x_i^{\top}\w^t) - y_i)\cdot\x_i,
\end{equation}
We note that the above perceptron algorithm is not the same as gradient descent algorithm. Assuming the differentiability, the standard chain rule computes the gradient of the $i^{\mathrm{th}}$ sample loss function by 
\begin{equation}\label{eq:gradient}
(\mbox{sign}(\x_i^{\top}\w^t) - y_i)\cdot (\mbox{sign})^\prime(\x_i^{\top}\w^t)\cdot\x_i.
\end{equation}
Comparing (\ref{eq:gradient}) with (\ref{eq:perceptron}), we observe that the perceptron algorithm essentially uses a coarse (and fake) gradient as if $(\mbox{sign})^\prime$ composited in the chain rule was the derivative of identity function being the constant 1. 
 
  The idea of STE was extended to train deep networks with binary activations \cite{hinton2012neural}. Successful experimental results have demonstrated the effectiveness of
the empirical STE approach. For example, \cite{bengio2013estimating} proposed a STE variant which uses the derivative of sigmoid function instead of identity function. 
\cite{bnn_16} used the derivative of hard tanh function, i.e., $1_{\{|x|\leq1\}}$, as an STE in training binarized neural networks. 
To achieve less accuracy degradation,
STE was later employed to train DNNs with quantized activations at higher bit-widths \cite{Hubara2017QuantizedNN,dorefa_16,halfwave_17,pact,yin2018blended}, where some other STEs were proposed including the derivatives of standard ReLU ($\max\{x, 0\}$) and clipped ReLU ($\min\{\max\{x, 0\}, 1\}$). 

Regarding the theoretical justification, it has been established that the perceptron algorithm in (\ref{eq:perceptron}) with identity STE converges and perfectly classifies linearly separable data; see for examples \cite{widrow199030,freund1999large} and references therein. 
Apart from that, to our knowledge, there had been almost no theoretical justification of STE until recently:  \cite{yin2018understanding} considered a two-linear-layer network with binary activation for regression problems. The training data is assumed to be instead linearly non-separable, being generated by some underlying model with true parameters. In this setting, \cite{yin2018understanding} proved that the working STE is actually non-unique and that the coarse gradient algorithm is descent and converges to a valid critical point if choosing the STE to be the proxy derivative of either ReLU (i.e., $\max\{x, 0\}$) or clipped ReLU function (i.e., $\min\{\max\{x, 0\}, 1\}$). Moreover, they proved that the identity STE fails to give a convergent algorithm for learning two-layer networks, although it works for single-layer perception.

\subsection{Main Contributions}
\pgfplotsset{every axis/.append style={
                    axis x line=middle,
                    axis y line=middle,
                    axis line style={->},
                    xlabel={$x$},
                    ylabel={$\sigma(x)$},
                    y label style={at={(0.1,1)}},
                    line width=1pt,},
                    % line style
                    cmhplot/.style={color=blue,mark=none},
                    soldot/.style={color=blue,only marks,mark=*},
                    holdot/.style={color=blue,fill=white,only marks,mark=*},
                    % framed
                    %framed/.style={axis background/.style ={draw=gray}},
                    }

% arrows
%\tikzset{>=stealth}
\begin{figure}[ht]
\centering
\scalebox{.8}{
\begin{tabular}{cc}
  \begin{tikzpicture}
        \begin{axis}[
                %framed,
                xmin=-0.45,xmax=1.2,
                ymin=-0.35,ymax=1.4,
                xtick={0,1},
                xticklabels={0,$\tau$},
                ytick={0,1},
                yticklabels={0,$\tau$},
            ]
        \addplot[cmhplot,domain=-1.5:0]{0};
        \addplot[cmhplot,domain=0:1]{1};
        \addplot[soldot]coordinates{(0,0)};
        \addplot[holdot]coordinates{(0,1)};
        \legend{1-bit Quantized ReLU}
        \end{axis}
    \end{tikzpicture}&
    \begin{tikzpicture}
        \begin{axis}[
                %framed,
                xmin=-1.5,xmax=4,
                ymin=-1,ymax=4,
                xtick={0,...,3},
                xticklabels={0,$\tau$,$2\tau$,$3\tau$},
                ytick={0,...,3},
                yticklabels={0,$\tau$,$2\tau$,$3\tau$},,
            ]
        \addplot[cmhplot,domain=-1.5:0]{0};
        \addplot[cmhplot,domain=0:1]{1};
        \addplot[cmhplot,domain=1:2]{2};
        \addplot[cmhplot,domain=2:3.5]{3};
        \addplot[soldot]coordinates{(0,0)(1,1)(2,2)};
        \addplot[holdot]coordinates{(0,1)(1,2)(2,3)};
        \addlegendentry{2-bit Quantized ReLU}
        \end{axis}
    \end{tikzpicture}
  
\end{tabular}
}
\caption{Quantized activation functions. $\tau$ is a value determined in the network training; see section 8.2.}
\label{qrelu}
\end{figure}
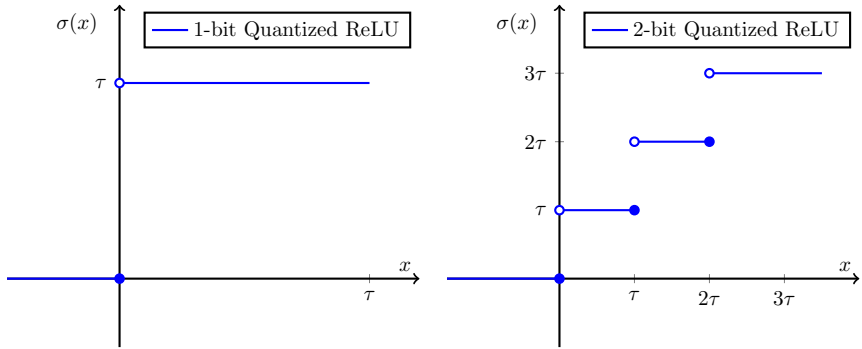
Fig. \ref{qrelu} shows examples of 1-bit (binary) and 2-bit (ternary) activations. We see that a quantized activation function zeros out any negative input, while being increasing on the positive half. Intuitively, a working surrogate of the quantized function used in backward pass should also enjoy this monotonicity, as conjectured by \cite{yin2018understanding} which proved the effectiveness of coarse gradient for two  specific STEs: derivatives of ReLU and clipped ReLU, and for binarized activation. In this work, we take a further step towards understanding the convergence of coarse gradient methods for training networks with {\it general quantized activations} and for {\it  classification of linearly non-separable data}. {\it A major analytical challenge we face here is that the network loss function is not in closed analytical form, in sharp contrast to \cite{yin2018understanding}.} We present more general results to provide meaningful guidance on how to choose STE in activation quantization. Specifically, we study multi-category classification of linearly non-separable data by a two-linear-layer network with multi-bit activations and hinge loss function.
We establish the convergence of coarse gradient methods for a broad class of surrogate functions. More precisely, if a function $g:\mathbb{R}\to\mathbb{R}$ satisfies the following properties:
\begin{itemize}
    \item $g(x) = 0$ for all $x\leq0$,
    \item $g'(x) \geq\delta>0$ for all $x>0$ with some constant $\delta$,
\end{itemize}
then with proper learning rate, the corresponding coarse gradient method converges and perfectly classifies the non-linear data when $g^\prime$ serves as the STE during the backward pass. This gives the affirmation of a conjecture in \cite{yin2018understanding} regarding good choices of STE  
%In addition, in comparison with \cite{yin2018understanding}, we consider 
for a {\it classification} (rather than regression) task under {\it weaker data assumptions, e.g. allowing non-Gaussian distributions}.

\subsection{Notations}
%For any positive integer $n$, we denote $[n]=\left\{1,2,\cdots,n\right\}$. For any finite dimensional linear space $V\subseteq \mathbb{R}^d$, we define $V^k$ to be the collection of matrices of form $\left[\bm x_1,\cdots,\bm x_k\right]\in\mathbb{R}^{d\times k}$, where $\bm x_j\in V$ is the $j$-th column vector. For any set $\mathcal{X}$, $\mathds{1}_{\mathcal{X}} (x) = 1$ if $x\in\mathcal{X}$ else $0$, is the indicator function of $\mathcal{X}$. For any vector $\bm x\in \mathbb{R}^d$, we denote $|\bm x|$ be the $\ell_2$ norm of $\bm x$. For a matrix $\W=[\bm w_1,\cdots,\bm w_k]\in\mathbb{R}^{d\times k}$, $|\W|:=\sum_{j=1}^k\left|\bm w_j\right|$ is the column-wise $\ell_2$-norm sum. $\mathcal{H}^d$ denotes the $d$-dim Hausdorff measure. Finally, for any non-zero vector $\x$, we denote $\tilde{\x}$ to be the unit vector on the same direction and $\tilde{\bm 0}=\bm 0$.
We have Table \ref{notations} for notations used in this paper. 

\begin{table}[ht]
\vspace{-0.2cm}
\caption{Frequently Used Notations}\label{notations}
\vspace{-0.3cm}
\centering
%\small
\begin{tabular}{ l|l }
\hline
{\bf Symbols} & {\bf Definitions} \\ \hline
$[n]$       & $\{1,2,\cdots,n\}$ \\  \hline
$\mathds{1}_\mathcal{S}(x)$
            & indicator function which take value $1$ for $x\in\mathcal{S}$\\ 
            & and $0$ for  $x\not\in\mathcal{S}$\\ \hline
$|\x|$      & $\ell_2$-norm of vector $\x$\\ \hline
$|\W|$      & the collumn-wise $\ell_2$-norm sum for a matrix $\W$. \\
            & For $\W:=[\w_1,\cdots,\w_k]$, $|\W|=\sum_{j=1}^k|\w_j|$\\\hline
$\mathcal{H}^d$
            & $d$-dimensional Hausdorff measure\\ \hline
$\tilde{\x}$
            & the unit vector in the direction of $\x$, i.e., $\tilde{\x}:=\frac{\x}{|\x|}$.\\
            & Additionally, $\tilde{\bm0}:=\bm0$.\\ \hline
$\sigma$      & quantized ReLU function \\ \hline
$\Omega_{\W}$
            & $\{\x\in\mathcal{X}: l(\W;\{\x,y\})>0\}$ \\ \hline
$\Omega_{\bm v}^a$
            & $\{\x\in\mathcal{X}: \langle\bm v,\x\rangle>a\}$ \\ \hline
$\Omega_\W^j$
            & $\Omega_\W\cap\Omega_{\bm w_j}^0$ \\ \hline
\end{tabular}
\vspace{-0.3cm}
\end{table}

\section{Problem Setup}
\subsection{Data Assumptions}
In this section, we consider the $n$-ary classification problem in the $d$-dimensional space $\mathcal{X}=\mathbb{R}^{d}$.
Let $\mathcal{Y}=[n]$ be the set of labels, and for $i\in[n]$ let $\mathcal{D}_i$ be probabilistic distributions over $\mathcal{X}\times\mathcal{Y}$. Throughout this paper, we make the following assumptions on the data:
%\begin{assumption}\label{data}
\begin{enumerate}
\item \textbf{(Separability)} There are $n$ orthogonal sub-spaces $V_i \subseteq\mathcal{X}$, $i\in[n]$ where $\dim V_i=d_i$, such that  $$\mathop{\mathbb{P}}_{\{\bm x,y\}\sim\mathcal{D}_i}\left[\bm x\in \mathcal{V}_i\text{ and }y=i\right]=1, \; \mbox{for all } i \in [n].$$
    \item \textbf{(Boundedness of data)}  There exist positive constants $m$ and $M$, such that $$\mathop{\mathbb{P}}_{\{\bm x,y\}\sim\mathcal{D}_i}\left[m<\left|\bm x\right|<M\right]=1, \; \mbox{for all } i \in [n].$$
    \item \textbf{(Boundedness of p.d.f.)} For $i\in[n]$, let $p_i$ be the marginal probability distribution function of $\mathcal{D}_i$ on $\mathcal{V}_i$. For any $\x \in \mathcal{V}_i$ with $m<\left|\bm x\right|<M$, it holds that 
    $$0< p_i( \x)<p_{\text{max}}<\infty.$$ 
\end{enumerate}
%\end{assumption}
Later on, we denote $\mathcal{D}$ to be the evenly mixed distribution of $\mathcal{D}_i$ for $i\in[n]$.
\begin{rmk}
The orthogonality of subspaces $\mathcal{V}_i$'s in the data assumption (1) above is technically needed for our proof here. However, the  convergence in Theorem \ref{main1} to a perfect classification with random initialization is observed in more general settings when $\mathcal{V}_i$'s form acute angles and contain a certain level of noise. We refer to section \ref{experiments} for  supporting experimental results. 
\end{rmk}
\begin{rmk}
Assumption (3) can be relaxed to the following, while the proof remains basically the same. 

$\mathcal{D}_i$ is a mixture of $n_i$ distributions namely $\mathcal{D}_{i,j}$ for $j\in[n_i]$. 
There exists a linear decomposition of $\mathcal{V}_i=\bigoplus_{j=1}^{n_i}\mathcal{V}_{i,j}$ and $\mathcal{D}_{i,j}$ each has a marginal probability distribution function $p_{i,j}$ on $\mathcal{V}_{i,j}$. For any $\x\in \mathcal{V}_{i,j}$ and $<m<|\x|<M$, it holds that
$$0<p_{i,j}(\x)\leq p_{\text{max}}<\infty.$$
\end{rmk}

\subsection{Network Architecture}
We consider a two-layer neural architecture with $k$ hidden neurons. Denote 
by $\W=\left[\bm w_1,\cdots,\bm w_{k}\right]\in \mathbb{R}^{d\times k}$ the weight matrix in the hidden layer. Let
$$h_j=\left\langle\bm w_j,\bm x\right\rangle$$
the input to the activation function, or the so-called pre-activation. 
\textbf{Throughout this paper, we make the following assumptions:}
%properties of the matrix $\bm V$}:
\begin{assumption}\label{v}
The weight matrix in the second layer $\bm V=[\bm v_1,\cdots,\bm v_n]$ is fixed and known in the training process and satisfies: 
\begin{enumerate}
    \item For any $i\in[n]$, there exists some $j\in[k]$ such that $v_{i,j}>0$.
    \item If $v_{i,j}>0$, then for any $r\in[n]$ and $r\not=i$, we have $v_{r,j}=0$.
    \item For any $i\in[n]$ and $j\in[k]$ we have $v_{i,j}<1$.
\end{enumerate} 
\end{assumption}
One can easily show that as long as $k\geq n$, such a matrix $\bm V=(v_{i,j})$ is ubiquitous.

\medskip

For any input data $\bm x\in\mathcal{X}=\mathbb{R}^{d}$, the neural net output is 
\begin{equation}\label{net1}
f(\W;\x)=[o_1,\cdots,o_n],
\end{equation}
where
$$o_i = \left\langle\bm v_i,\sigma\left(\bm h\right)\right\rangle=\sum_{j=1}^kv_{i,j}\sigma(h_j). 
$$

%$\bm v_i\in\mathbb{R}^k$
%is prescribed constant vector independent of the training process. 
The  $\sigma(\cdot)$ is the quantized ReLU function acting element-wise; 
see Fig. \ref{qrelu} for examples of binary and ternary activation functions. More general quantized ReLU function of the bit-width $b$ can be defined as follows: 
$$
\sigma(x)=\begin{cases}
     0 & \text{if} \quad x\leq0, \\
     \text{ceil}(x) & \text{if} \quad 0<x<2^b-1, \\
     2^b -1 & \text{if} \quad x\geq 2^b-1. \\
\end{cases}
$$
The prediction is given by the network output label
$$\hat{y}(\W,\x)=\mathop{\text{argmax}}_{r\in[n]}o_r,$$
ideally $\hat{y}(\x)=i$ for all $\x\in \mathcal{V}_i$. The classification accuracy in percentage is the frequency that this event  occurs (when network output label $\hat{y}$ matches the true label) on a validation data set. 

Given the data sample $\{\x, y\}$, the associated hinge loss function reads 
\begin{equation}\label{sample_loss}
l(\W; \{ \x, y \}) :=  \max\left\{0, 1 - f_y\right\}:=\max\left\{0, 1 - \left(o_y-\max_{i\not=y}o_i\right)\right\}.
\end{equation}
To train the network with quantized activation $\sigma$, we consider the following population loss minimization problem
\begin{equation}\label{loss}
\min_{\W\in\mathbb{R}^{d\times k}}\; l\left(\W\right) := \mathop{\mathbb{E}}_{\{ \x, y \}\sim\mathcal{D}}\left[ l\left(\W; \{\x, y\}\right)\right],
\end{equation}
where the sample loss $l\left(\W; \{\x, y\}\right)$ is defined in (\ref{sample_loss}).
Let $l_i$ be the population loss function of class $i$ with the label $y=i$, $i \in [n]$. More precisely, 
$$\begin{aligned}
l_i(\W)=&\mathop{\mathbb{E}}_{\{\bm x,y\}\sim\mathcal{D}_i}\left[\max\left\{0, 1 - f_i\right\}\right]\\
=&\mathop{\mathbb{E}}_{\{\bm x,y\}\sim\mathcal{D}_i}\left[\max\left\{0, 1 - \left(o_i - \max_{r\not=i}o_r\right)\right\}\right].
\end{aligned}$$
Thus, we can rewrite the loss function as 
$$l(\W)=\frac{1}{n}\sum_{i=1}^nl_i(\W).$$
Note that the population loss
$$l_i(\W)=\mathop{\mathbb{E}}_{\{\x,y\}\sim\mathcal{D}_i}\left[l(\W;\{\x,y\})\right]$$
fails to have simple closed-form solution even if $p_i$ are constant functions on their supports. We do not have closed-form formula at hand to analyze the learning process, which makes our analysis challenging.

For notational convenience, we define:
$$\Omega_{\W}=\left\{\x\in\mathcal{X}: l(\W;\{\x,y\})>0\right\},$$
$$\Omega_{\bm v}^a=\left\{\x\in\mathcal{X}:\left\langle\bm v,\bm x\right\rangle>a\right\},$$
and
$$\Omega_{\W}^j=\Omega_{\W}\cap\Omega_{\bm w_j}^0.$$

\subsection{Coarse Gradient Methods}
We see that derivative of quantized ReLU function $\sigma$ is a.e. zero, which gives a trivial gradient of sample loss function with respect to (w.r.t.) $\bm w_j$. Indeed, differentiating the sample loss function with respect to $\bm w_j$, we have
$$\nabla_{\bm w_j} l(\W;\{\x,y\})=-\left(v_{y,j}-v_{\xi,j}\right)\,\mathds{1}_{\Omega_{\W}}(\x)\,\sigma'\left(h_j\right)\x = \mathbf{0}, \mbox{ a.e.}, \quad 1\leq j\leq k$$
where $\xi=\mathop{\text{argmax}}_{i\not=y}o_i$.

The partial coarse gradient w.r.t. $\w_j$ associated with the sample $\{\x, y\}$ is given by replacing $\sigma'$ with a straight through estimator (STE) which is the derivative of function $g$, namely,
\begin{equation}\label{cgrad}
\tilde{\nabla}_{\bm w_j}l(\W;\{\x,y\}) := -\left(v_{y,j}-v_{\xi,j}\right)\,\mathds{1}_{\Omega_{\W}}(\x)\,g'(h_j)\x.
\end{equation}
The sample coarse gradient $\tilde{\nabla}l(\W;\{\x,y\})$ is just the concatenation of $\tilde{\nabla}_{\bm w_j}l(\W;\{\x,y\})$'s.
It is worth noting that coarse gradient is not an actual gradient, but some biased first-order oracle which depends on the choice of $g$. 

%There are many reasons why we use coarse gradient instead of actual gradient. %First, the actual gradient is difficult to calculate. Since $\sigma$ is piecewice constant, we need infinite sample data in order to estimate the real gradient. In contrast, coarse gradient can be approximated by finite many sample data. Second, we shall show, in our setup, that the critical points in sense of actual and coarse gradient coincide with each other. The results of this paper works for a collection of STEs. 

\textbf{Throughout this paper, we consider a class of surrogate functions during the backward pass with the following properties:}
\begin{assumption}\label{g}
$g:\mathbb{R}\to\mathbb{R}$ satisfies 
\begin{enumerate}
    \item $g(x) = 0$ for all $x\leq0$.
    %\item $g\in C^1\left(\left(0,\infty\right)\right)$.
    \item $g'(x)\in[\delta,\tilde\delta]$ for all $x>0$ with some constants $0<\delta<\tilde\delta<\infty$. 
\end{enumerate} 
\end{assumption}
Such a $g$ is ubiquitous in quantized deep networks training; see Fig.\ref{ste} for examples of $g(x)$  satisfying Assumption \ref{g}. Typical examples include the classical ReLU $g(x) = \max(x, 0)$ and log-tailed ReLU \cite{halfwave_17}:
$$g(x)=\left\{\begin{array}{ccc}
     0& \text{if} &x\leq0\\
     x& \text{if} &0<x\leq q_b\\
     q_b+\log (x-q_b+1)& \text{if} & x>q_b\\
\end{array}\right.$$
where $q_b := 2^b-1$ is the maximum quantization level.
In addition, if the input of the activation function is bounded by a constant, one also can use $g(x)=\max\{0,q_b (1- e^{-x/q_b})\}$, which we call \emph{reverse exponential STE}. 

\pgfplotsset{every axis/.append style={
                    axis x line=middle,
                    axis y line=middle,
                    axis line style={->},
                    xlabel={$x$},
                    ylabel={$g(x)$},
                    y label style={at={(0.1,1)}},
                    line width=1pt,},
                    % line style
                    cmhplot/.style={color=blue,mark=none},
                    soldot/.style={color=blue,only marks,mark=*},
                    holdot/.style={color=blue,fill=white,only marks,mark=*},
                    % framed
                    %framed/.style={axis background/.style ={draw=gray}},
                    }

% arrows
%\tikzset{>=stealth}
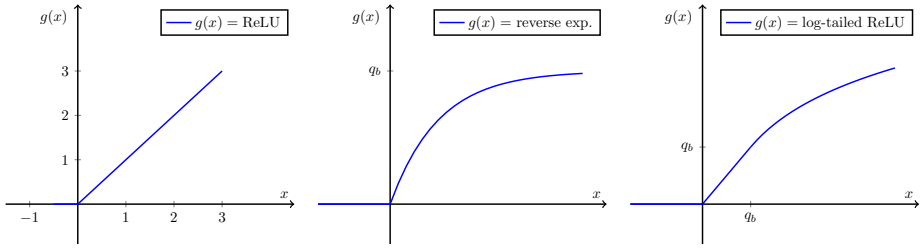
\begin{figure}[ht]
\centering
\scalebox{.7}{
\begin{tabular}{ccc}
\begin{tikzpicture}[scale=0.8]
        \begin{axis}[
                %framed,
                xmin=-1.5,xmax=4.5,
                ymin=-0.9,ymax=4.5,
                xtick={-1,...,3},
                ytick={0,...,3},
            ]
        \addplot[cmhplot,domain=-0.5:0]{0};
        \addplot[cmhplot,domain=0:3]{x};
        \addlegendentry{$g(x)= \mbox{ReLU}$}
        \end{axis}
    \end{tikzpicture}&
  \begin{tikzpicture}[scale=0.8]
        \begin{axis}[
                %framed,
                xmin=-1.5,xmax=4.5,
                ymin=-0.3,ymax=1.5,
                xtick={0},
                ytick={0,...,1},
                yticklabels={0,$q_b$}
            ]
        \addplot[cmhplot,domain=-1.5:0]{0};
        \addplot[cmhplot,domain=0:4]{1-exp(-x)};
        \addlegendentry{$g(x)=$ \mbox{reverse exp.}}
        \end{axis}
    \end{tikzpicture}&

    \begin{tikzpicture}[scale=0.8]
        \begin{axis}[
                %framed,
                xmin=-1.5,xmax=4.5,
                ymin=-0.7,ymax=3.5,
                xtick={0,1},
                xticklabels={0,$q_b$},
                ytick={0,1},
                yticklabels={0,$q_b$}
            ]
        \addplot[cmhplot,domain=-1.5:0]{0};
        \addplot[cmhplot,domain=0:1]{x};
        \addplot[cmhplot,domain=1:4]{ln(x)+1};
        \legend{$g(x) = \mbox{log-tailed ReLU}$}
        \end{axis}
    \end{tikzpicture}
  
\end{tabular}}
\caption{Different choices of $g(x)$ for the straight-through estimator.}
\label{ste}
\end{figure}

To train the network with quantized activation $\sigma$, we use the expectation of coarse gradient over training samples:
$$
\tilde{\nabla} l(\W): = \mathop{\mathbb{E}}_{\{ \x, y \}\sim\mathcal{D}} \tilde{\nabla}l(\W;\{\x,y\})
$$
where $\tilde{\nabla}l(\W;\{\x,y\})$ is given by (\ref{cgrad}). In this paper, we study the convergence of coarse gradient algorithm for solving the minimization problem (\ref{loss}), which takes the following iteration with some learning rate $\eta>0$:
\begin{equation}\label{cgd}
\W^{t+1}= \W^{t}-\eta\,\tilde{\nabla} l(\W^{t})
\end{equation} 

\section{Main Result and Outline of Proof}
We show that if the iterates $\{\W^t\}$ are uniformly bounded in $t$, coarse gradient decent with the proxy function $g$ under Assumption \ref{g} converges to a global minimizer of the population loss, resulting in a perfect classification.

\begin{theorem}\label{main1}
Suppose data assumptions (1)-(3)
and STE assumptions  \ref{v}-\ref{g} hold. If the network initialization satisfies $\w_{j,i}^0\not=0$ for all $j\in[k]$ and $i\in[n]$ and $\W^t$ is uniformly bounded by $R$ in $t$, then for all $v_{i,j}>0$ we have
$$\lim_{t\rightarrow\infty}\left|\tilde{\nabla}_{\bm w_j}l_i(\W^t)\right|=0.$$

Furthermore, if $\W^\infty$ is an accumulation point of $\{\W^t\}$ and all non-zero unit vectors $\tilde{\bm w}_{j,i}^\infty$'s are distinct for all $j\in[k]$ and $i\in[n]$, then
$$\mathop{\mathbb{P}}_{\{\x,y\}\sim\mathcal D}\left(\hat{y}\left(\W^{\infty},\x\right)\neq y\right)=0.$$
\end{theorem}

We outline the major steps in the proof below. \medskip

\textbf{Step 1: Decompose the population loss into $n$ components.} Recall the definition of $l_i$ which is population loss functions for $\{\x,y\}\sim\mathcal{D}_i$. In Section 4, we show under certain decomposition of $\W$, the coarse gradient decent of each one of them only affects a corresponding component of $\W$. 

\textbf{Step 2: Bound the total increment of weight norm from above.} 
Show that for all $v_{i,j}>0$ we have $|\bm w_{j,i}|$'s 
are monotonically increasing under coarse gradient descent. 
Based on boundedness on $\W$, we further give an upper bound on the total increment of all $|\bm w_j|$'s, from which the convergence of coarse gradient descent
follows.

\textbf{Step 3: Show that when the coarse gradient vanishes, so does the population loss.} 
%The previous step can ensure the convergence of coarse gradient decent. 
In section 6, we show that when the coarse gradient vanishes towards the end of  training, the population loss is zero which implies a  perfect classification. 

% \textbf{Step 4: Conclusion.} With preparation of the previous three steps, we combine all the results to complete the proof of Theorem \ref{main1}.

\section{Space Decomposition}

With $\mathcal{V}=\bigoplus_{i=1}^n \mathcal{V}_i$, we have the orthogonal complement of $\mathcal{V}$ in $\mathcal{X}=\mathbb{R}^d$, namely $\mathcal{V}_{n+1}$. Now, we can decompose $\mathcal{X}=\mathbb{R}^d$ into $n+1$ linearly independent parts:
$$\mathbb{R}^d=\mathcal{V}\bigoplus \mathcal{V}_{n+1}=\bigoplus_{i=1}^{n+1}\mathcal{V}_i$$
and for any vector $\bm w_j\in\mathbb{R}^d$, we have a unique decomposition of $\bm w_j$:
$$\w_j=\sum_{i=1}^{n+1}\w_{j,i},$$
where $\w_{j,i}\in \mathcal{V}_i$  for $i\in[n+1]$. 
To simply notation, we let
$$\W_i=\left[\w_{1,i},\cdots,\w_{k,i}\right].$$
\begin{lemma}\label{remainder}
For any $\W\in\mathbb{R}^{k\times d}$ and $i\in[n]$, we have
$$l_i\left(\W\right)=l_i\left(\sum_{r=1}^n\W_r\right)=l_i(\bm W_i).$$
\end{lemma}
\begin{proof}
%{Proof of Lemma \ref{remainder}}
Note that for any $\x\in \mathcal{V}_i$ and $j\in[k]$, we have $\x\in \mathcal{V}$, so
$$\left\langle\w_{j,n+1},\x\right\rangle=0$$
and 
$$h_j=\left\langle\w_{j},\x\right\rangle=\left\langle\sum_{j=1}^k\w_{j,i},\x\right\rangle=\left\langle\bm w_{j,i},\x\right\rangle.$$
Hence
$$f\left(\W;\x\right)=f\left(\sum_{j=1}^k\W_i;\x\right)=f\left(\bm W_i\right)$$
for all $\W\in\mathbb{R}^{d\times k}$, $\x\in \mathcal{V}_i$. The desired result follows.
\end{proof}
\begin{lemma}\label{decouple}
Running the algorithm (\ref{cgd}) on $l_i$ only does not change the value of $\W_r$ for all $r\not=i$. More precisely, for any $\W\in\mathbb{R}^{d\times k}$, let $$\W'=\W-\eta\tilde{\nabla}l_i(\W),$$ then for any $r\in[n]$ and $r\not=i$ $$\W_r'=\W_r.$$
\end{lemma}
\begin{proof}[Proof of Lemma \ref{decouple}]
Assume $i,r\in[n]$ and $i\not=r$. Note that 
$$\w_j'=\w_j-\eta\tilde{\nabla}_{\bm w_j}l_i(\W)$$
and
$$\tilde{\nabla}_{\bm w_j}l_i(\W)=-\mathop{\mathbb{E}}_{\{\x,y\}\sim\mathcal{D}_i}\left[\left(v_{y,j}-v_{\xi,j}\right)\,\mathds{1}_{\Omega_{\W}}(\x)\,g'(h_j)\x\right]\in V_i.$$
Since $\mathcal{V}_i$'s are linearly  independent, we have
$$\w_{j,i}'=\w_{j,i}-\eta\tilde{\nabla}_{\bm w_j}l_i(\W)$$
and
$$\bm w_{j,r}'=\bm w_{j,r}.$$
\end{proof}

By the above result, we know (\ref{cgd}) is equivalent to
\begin{equation}\label{dcgd}%decoupled cgd
\W_i^{t+1} = \W_i^t - \frac{\eta}{n} \tilde{\nabla} l_i\left(\W^t\right).
\end{equation}

\section{Learning Dynamics}
In this section, we show that some components of the weight iterates have strictly increasing magnitude whenever coarse gradient does not vanish, and it quantifies the increment during each iteration. 
\begin{lemma}\label{helper1}
Assume $$\hat{v}_j=\max_{i_1,i_2\in[n]}v_{i_1,j}-v_{i_2,j}\,,$$ we have the following estimate:
$$\mathop{\mathbb{P}}_{\{\x,y\}\sim\mathcal{D}_i}\left(\Omega_W^j\right)\geq\frac{1}{\hat{v}_j\tilde{\delta}M}\left|\tilde{\nabla}_{\bm w_j}l_i\left(\W\right)\right|.$$
\end{lemma}
\begin{proof}[Proof of Lemma \ref{helper1}]
\begin{align*}
\left|\tilde{\nabla}_{\bm w_j}l_i(\W)\right|
=&\left|\mathop{\mathbb{E}}_{\{\x,y\}\sim\mathcal{D}_i}\left[\left(v_{y,j}-v_{\xi,j}\right)\,\mathds{1}_{\Omega_{\W}}(\x)\,g'(h_j)\x\right]\right|\\
\leq& \hat{v}_j\tilde{\delta}M\mathop{\mathbb{E}}_{\{\x,y\}\sim\mathcal{D}_i}\left[\mathds{1}_{\Omega^j_{\W}}(\x)\right]\\
= &\hat{v}_j\tilde{\delta}M\mathop{\mathbb{P}}_{\{\x,y\}\sim\mathcal{D}_i}\left(\Omega_{\W}^j\right)
\end{align*}
\end{proof}
\begin{lemma}\label{helper2}
For any $j\in[k]$ if 
$$\tilde{v}_{i,j}:=v_{i,j}-\max_{r\not=i}v_{r,j}>0$$
we have
$$\left\langle\tilde{\w}_{j,i},-\tilde{\nabla}_{\w_j}l_i(\W)\right\rangle\geq\frac{\tilde{v}_{i,j}\delta}{2C_p}\mathop{\mathbb{P}}_{\{\x,y\}\sim\mathcal{D}_i}\left(\Omega_W^j\right)^2,$$
where
$$C_p=\max_{\bm v\in V_i,a\in\mathbb{R}}\int_{\langle\bm v,
\bm x\rangle=a}p_i(\x)\;d\,\mathcal{H}^{d_i-1}(\x).$$
\end{lemma}
\begin{proof}[Proof of Lemma \ref{helper2}]
First, we prove an inequality which will be used later. Recall that $|\x|\leq M$, and that $\tilde{\nabla}_{\bm w_j}l(\W,\{\x,y\})\not=0$ only when $\bm x\in\Omega_{\W}^j$. Hence, we have $\left\langle\tilde{\bm w}_{j,i},\x\right\rangle>0$.
We have
\begin{align*}
\mathop{\mathbb{P}}_{\{\x,y\}\sim\mathcal{D}_i}\left(\Omega_{\W}^j\cap\left\{\x:\left\langle\tilde{\w}_{j,i},\x\right\rangle<t\right\}\right)
=&\int_{\Omega_\W^j}\mathds{1}_{\left\{\left\langle\tilde{\w}_{j,i},\x\right\rangle<t\right\}}(\x)p_i(\x)\;d\,\x\\
=&\int_0^t\int_{\left\langle\tilde{\w}_{j,i},\x\right\rangle=s}p_i(\x)\;d\,\mathcal{H}^{d_i-1}(\x)\;d\,s \\
\leq & t\ C_p.
\end{align*}

Now, we use Fubini's Theorem to simplify the inner product:
\begin{align*}
\left\langle\tilde{\bm w}_{j,i},-\tilde{\nabla}_{\bm w_j}l_i(\W)\right\rangle
=&\mathop{\mathbb{E}}_{\{\x,y\}\sim\mathcal{D}_i}\left[\left(v_{y,j}-v_{\xi,j}\right)\mathds{1}_{\Omega_{\W}^j}(\x)\,g'(h_j)\,\langle\tilde{\bm w}_{j,i},\x\rangle\right]\\
\geq&\tilde{v}_{i,j}\,\delta\int_{\Omega_{\W}^j\cap V_i}\langle\tilde{\bm w}_{j,i},\x\rangle p_i(\x)\;d\,\x\\
=&\tilde{v}_{i,j}\,\delta\int_{\Omega_{\W}^j\cap V_i}\int_0^\infty\mathds{1}_{\left\{\langle\tilde{\w}_{j,i},\x\rangle>t\right\}}\;d\,t\;p_i(\x)\;d\,\x\\
=&\tilde{v}_{i,j}\,\delta\int_0^\infty\int_{\Omega_{\W}^j\cap V_i}\mathds{1}_{\left\{\langle\tilde{\w}_{j,i},\x\rangle>t\right\}}\;p_i(\x)\;d\,\x\;d\,t\\
=&\tilde{v}_{i,j}\,\delta\int_0^\infty\mathop{\mathbb{P}}_{\{\x,y\}\sim\mathcal{D}_i}\left(\Omega_{\W}^j\cap\left\{\x:\langle\tilde{\bm w}_{j,i},\x\rangle>t\right\}\right)d\,t.
\end{align*}

Now using the inequality just proved above, we have

\begin{align*}
&\mathop{\mathbb{P}}_{\{\x,y\}\sim\mathcal{D}_i}\left(\Omega_{\W}^j\cap\left\{\x:\langle\tilde{\bm w}_{j,i},\x\rangle>t\right\}\right)\\
=&\mathop{\mathbb{P}}_{\{\x,y\}\sim\mathcal{D}_i}\left(\Omega_\W^j\right)-\mathop{\mathbb{P}}_{\{\x,y\}\sim\mathcal{D}_i}\left(\Omega_{\W}^j\cap\left\{\x:\langle\tilde{\bm w}_{j,i},\x\rangle<t\right\}\right)\\
\geq&\max\left\{\mathop{\mathbb{P}}_{\{\x,y\}\sim\mathcal{D}_i}\left(\Omega_\W^j\right)-t\;C_p,0\right\}.
\end{align*}

Combining the above two inequalities, we have
\begin{align*}
\left\langle\tilde{\bm w}_{j,i},-\tilde{\nabla}_{\bm w_{j}}l_i(\W)\right\rangle
\geq&\tilde{v}_{i,j}\,\delta\int_0^\infty\max\left\{\mathop{\mathbb{P}}_{\{\x,y\}\sim\mathcal{D}_i}\left(\Omega_\W^j\right)-t\;C_p,0\right\}\;d\,t\\
\geq&\frac{\tilde{v}_{i,j}\,\delta}{2C_p}\mathop{\mathbb{P}}_{\{\x,y\}\sim\mathcal{D}_i}\left(\Omega_{\W}^j\right)^2.
\end{align*}
\end{proof}

\begin{lemma}\label{descent}
If $\tilde{v}_{i,j}>0$ in Lemma \ref{helper2}, then $\{|\bm w_{j,i}^{t}|\}$ in Equation (\ref{net1}) is non-decreasing with coarse gradient decent (\ref{cgd}).
Moreover, under the same assumption, we have
$$\left|\bm w_{j,i}^{t+1}\right|-\left|\bm w_{j,i}^t\right|\geq\frac{
\eta\tilde{v}_{i,j}\delta}{2nC_p \hat{v}_j^2\tilde{\delta}^2M^2}\left|\tilde{\nabla}_{\bm w_j}l_i(\W^t)\right|^2,$$
where $C_p$ is defined as in Lemma \ref{helper2} and $\hat{v}_j$ as in Lemma \ref{helper1}.
\end{lemma}

\begin{proof}[Proof of Lemma \ref{descent}]
Since
$\bm w_{j,i}^{t+1}=\bm w_{j,i}^{t}-\frac{\eta}{n}\tilde{\nabla}_{\bm w_j}l_i(\W^t)$,
we have
$$\left|\bm w_{j,i}^{t+1}\right|-\left|\bm w_{j,i}^t\right|\geq\left\langle\bm w_{j,i}^{t+1}-\bm w_{j,i}^t,\tilde{\bm w}_{j,i}^{t}\right\rangle=\left\langle-\frac{\eta}{n}\tilde{\nabla}_{\bm w_j}l_i(\W^t),\tilde{\bm w}_{j,i}^t\right\rangle.$$
Hence, it follows from  Lemma \ref{helper1} and Lemma  \ref{helper2} that
\begin{equation}\label{delta2}
\left|\bm w_{j,i}^{t+1}\right|-\left|\bm w_{j,i}^t\right|\geq\frac{
\eta\tilde{v}_{i,j}\delta}{2nC_p \hat{v}_j^2\tilde{\delta}^2M^2}\left|\tilde{\nabla}_{\bm w_j}l_i(\W^t)\right|^2,
\end{equation}
which is the desired result.
\end{proof}
Note that one component of $\bm w_j$ is increasing but the weights are bounded by assumption, hence, summation of the increments over all steps should also be bounded. This gives the following proposition:
\begin{prop}\label{conv1}
Assume $\{|\bm w_j^t|\}$ is bounded by $R$, then if $\tilde{v}_{i,j}>0$ in Lemma \ref{helper2}, then
$$\sum_{t=1}^\infty\left|\tilde{\nabla}_{\bm w_j}l_i(\W^t)\right|^2\leq\frac{2nC_p \hat{v}_j^2\tilde{\delta}^2M^2R}{
\eta\tilde{v}_{i,j}\delta}<\infty,$$
where $C_p$ is as defined in Lemma \ref{helper2} and $\hat{v}_j$ defined in Lemma \ref{helper1}. This implies that $$\lim_{t\to\infty} \left|\tilde{\nabla}_{\w_j}l_i(\W^t)\right| = 0$$
as long as $\tilde{v}_{i,j}>0$.
\end{prop}
\begin{rmk}
Lemmas \ref{helper1}, \ref{helper2}, \ref{descent} and Proposition \ref{conv1} were proved without Assumption \ref{v}. Under Assumption \ref{v}, we have $\hat{v}_j=\max_{i\in[n]}v_{i,j}$ in Lemma \ref{helper1} and 
$\tilde{v}_{i,j}=\hat{v}_j$ if $v_{i,j}>0$ and $\tilde{v}_{i,j}=-\hat{v}_j$ if $v_{i,j}=0$ in Lemma \ref{helper2}. 
\end{rmk}
\section{Landscape Properties}
We have shown that under boundedness assumptions, the algorithm will converge to some point where the coarse gradient vanishes. However, this does not immediately indicate the convergence to a valid point because coarse gradient is a fake gradient. We will need the following lemma to prove Proposition \ref{globalmin}, which confirms that the points with zero coarse gradient are indeed global minima. 

\begin{lemma}\label{top}
Let $\Omega=\left\{\x\in \mathbb{R}^l:m<|\x|<M\right\}$, where $0<m<M<\infty$. For $j\in[k]$, let $\Omega_j=\left\{\x:\langle\w_j,\x\rangle>a\right\}$, where $a\geq0$ and $\Omega_i\not=\Omega_j$ for all $i\not=j$. If for $i\in[k]$ and $\x\in\Omega_i\cap\Omega$, there exists some $j\not=i$ such that $\x\in\Omega_j$, then
$$\left(\mathop{\cup}_{j=1}^k\Omega_{j}\right)\cap\Omega=\emptyset\ \text{ or }\ \Omega.$$
\end{lemma}
\begin{proof}[Proof of Lemma \ref{top}]
Define $\tilde{\Omega}=\bigcup_{j=1}^k\Omega_j$, by De Morgan's law, we have
$$\tilde{\Omega}^c=\left(\mathop{\cup}_{j=1}^k\Omega_j\right)^c=\mathop{\cap}_{j=1}^k\Omega_j^c.$$
Note that $k$ is finite and $\bm0\in\Omega_j^c$ for all $j\in[k]$, we know $\tilde{\Omega}^c$ is a generalized polyherdon and hence either
$$\left(\partial\tilde{\Omega}\right)\cap\Omega=\emptyset$$
or
$$\mathcal{H}^{l-1}\left(\left(\partial\tilde{\Omega}\right)\cap\Omega\right)>0.$$
The first case is trivial. We show that the second case contradicts our assumption. Note that
$$\partial\tilde{\Omega}=\partial\left(\mathop{\cup}_{j=1}^k\Omega_j\right)\subseteq\mathop{\cup}_{j=1}^k\partial\Omega_j,$$
we know there exists some $j^\star\in[k]$ such that
$\mathcal{H}^{l-1}\left(\partial\Omega_{j^\star}\cap\Omega\right)>0.$
It follows from our assumption that 
$\tilde{\Omega}=\mathop{\cup}_{j=1}^k\Omega_j=\mathop{\cup}_{j\not=j^\star}\Omega_j$,
and hence
$$\mathcal{H}^{l-1}\left(\partial\Omega_{j^\star}\cap\partial\Omega_j\right)>0.$$
Note that $\partial\Omega_j$'s are hyperplanes. Therefore, 
$\Omega_j=\Omega_{j^\star}$, contradicting with our assumption that all $\Omega_j$'s are distinct. 
\end{proof}
The following result shows that the coarse gradient vanishes only at a global minimizer with zero loss, except for some degenerate cases. 
\begin{prop}\label{globalmin}
Under Assumption \ref{v}, if $\tilde{\nabla}_{\bm w_j} l_i(\W)=\bm 0$ for all $\tilde{v}_{i,j}>0$ and $\tilde{\bm w}_{j,i}$'s are distinct, then $l_i(\W)=0$. 
\end{prop}
\begin{proof}[Proof of Proposition \ref{globalmin}]
For quantized ReLU function, let $q_b :=\max\limits_{x\in\mathbb{R}}\sigma(x)$ be the maximum quantization level, so that $$\sigma(x)=\sum_{a=0}^{q_b-1}\mathds{1}_{\{x>a\}}(x).$$ 

Note that 
\begin{align*}
f_i\left(\W;\x\right)=o_i-o_\xi=\sum_{j=1}^k\left(v_{i,j}-v_{\xi,j}\right)\sigma(h_j)
=\sum_{j=1}^k\left(v_{i,j}- v_{\xi,j}\right)\sum_{a=0}^{q_b}\mathds{1}_{\Omega_{\bm w_j}^a}(\x).
\end{align*}

By assumption, $\tilde{\nabla}_{\bm w_j} l_i(\W)=\bm 0$ for all $\tilde{v}_{i,j}>0$ which implies $\mathds{1}_{\Omega_{\W}}(\x)\mathds{1}_{\Omega_{\bm w_j}^a}(\x)=0$ for all $\tilde{v}_{i,j}>0$ and $a\in[n]$ almost surely. Now, for any $\x\in\Omega_{\bm w_j}^a$ we have $\x\not\in\Omega_{\W}$. Note that $\x\in\Omega_{\W}$ if and only if $o_i-o_\xi\geq1$, then for any $\bm x\in\Omega_{\bm w_j}^a$, since $v_{i,j}-v_{\xi,j}<1$, there exist $j'\not=j$ and $a'\in[n]$ such that $v_{i,j'}>0$ and $\x\in\Omega_{\bm w_{j'}}^{a'}$. By Lemma \ref{top}, $\mathop{\mathbb{P}}_{\{\x,y\}\sim\mathcal{D}_i}\left[\Omega_{\W}\right]=0$ is empty, and thus $l_i(\W)=0$.
\end{proof}
The following lemma shows that the expected coarse gradient is continuous except at $\bm w_{j,i}=\bm0$ for some $j\in[k]$.

\begin{lemma}\label{contgrad}
Consider the network in (\ref{net1}). $\tilde{\nabla}_{\bm w_j}l_i(\W)$ is continuous on $$\left\{\W\in \mathbb{R}^{k\times d}:|\bm w_{j,i}|>0\text{ for all }j\in[k],i\in[n]\right\}.$$
\end{lemma}
\begin{proof}[Proof of Lemma \ref{contgrad}]
It suffices to prove the result for $j\in[k]$. Note that
$$
\tilde{\nabla}_{\bm w_j}l_i(\W)
=\mathop{\mathbb{E}}_{\{\bm x,y\}\sim\mathcal{D}_i}\left[-\left(v_{y,j}-v_{\xi,j}\right)\,\mathds{1}_{\Omega_{\W}}(\x)\,g'(h_j)\x\right]
$$
For any $\W^0$ satisfying our assumption, we know
$$\lim_{\W\rightarrow\W^0}\mathds{1}_{\Omega_\W}(\x)g'(h_j)=\mathds{1}_{\Omega_{\W^0}}(\x)g'(h_j^0), \mbox{ a.e.}$$ The desired result follows from the Dominant Convergence Theorem. 
\end{proof}

\section{Proof of Main Results}
Equipped with the technical lemmas, we present:  
\begin{proof}[Proof of Theorem \ref{main1}]
It is easily noticed from Assumption \ref{v} that $v_{i,j}>0$ if and only if $\tilde{v}_{i,j}>0$. 
By Lemma \ref{descent}, if $v_{i,j}>0$ and $|\bm w_{j,i}^0|>0$, then $|\bm w_{j,i}^t|>0$ for all $t$. Since $\W$ is randomly initialized, we can ignore the possibility that $\bm w_{j,i}^0=\bm0$ for some $j\in[k]$ and $i\in[n]$.
Moreover,
Proposition \ref{conv1} and Equation (\ref{cgd}) imply for all $v_{i,j}>0$ $$\lim_{t\rightarrow\infty}\left|\tilde{\nabla}_{\bm w_j}l_i(\W^t)\right|=0.$$

Suppose $\W^\infty$ is an accumulation point and $\bm w_{j,r}^{\infty}\not=\bm0$ for all $j\in[k]$ and $r\in[n]$, we know for all $v_{i,j}>0$
$$\tilde{\nabla}_{\bm w_j}l_i\left(\W^\infty\right)=\bm0.$$

Next, we consider the case when $\bm w_{j,r}=\bm0$ for some $j\in[k]$ and $r\in[n]$.
Lemma \ref{helper2} implies $v_{r,j}=0$.
We construct a new sequence $$\hat{\bm w}_{j,r}^t=\left\{
\begin{aligned}
\bm w_{j,r}^t& \;\; \text{ if }\bm w_{j,r}^\infty\not=0\\
\bm0&\;\; \text{ if }\bm w_{j,r}^\infty=0\\
\end{aligned}
\right.$$
and $$\hat{\bm W}_r^t=\left[\hat{\bm w}_{1,r}^t,\cdots,\hat{\bm w}_{k,r}^t\right].$$
With
$$\hat{o}_r=\sum_{j=1}^kv_{r,j}\sigma(\hat{h}_j)=\sum_{j=1}^kv_{r,j}\sigma\left(\left\langle\hat{\bm w}_{j,r},\x\right\rangle\right),$$
we know
$\hat{o}_r=o_r$ for all $r\in[n]$.
Hence, we have
$$l\left(\hat{\bm W}^t,\{\x,i\}\right)=\text{ReLU}\left(1-\hat{o}_i+\hat{o}_\xi\right)= l\left(\bm W^t,\{\x,i\}\right).$$
This implies that  $\Omega_{\hat{\W}^t}=\Omega_{\W^t}$,
so we have for all $j\in[k]$,
$$
\left|\left\langle\tilde{\nabla}_{\bm w_j}l_i(\hat{\bm W}_1^t),\tilde{\w}_{j,i}^t\right\rangle\right|\leq\left|\left\langle\tilde{\nabla}_{\bm w_j}l_i(\bm W_i^t),\tilde{\w}_{j,i}^t\right\rangle\right|\leq\left|\tilde{\nabla}_{\bm w_j}l_i(\bm W_i^t)\right|.$$
Letting $t$ go to infinity on both side, we get
$$\left|\left\langle\tilde{\nabla}_{\bm w_j}l_i(\hat{\bm W}^\infty),\tilde{\w}_{j,i}^\infty\right\rangle\right|=0.$$
By Lemma \ref{helper1} and Lemma \ref{helper2}, we know
$$\tilde{\nabla}_{\bm w_j}l_i(\W^\infty)=\tilde{\nabla}_{\bm w_j}l_i(\bm W_i^\infty)=0,$$
so $\tilde{\nabla}_{\W}l_i(\W^\infty)=0.$ By Proposition \ref{globalmin}, $l_i(\W^t)=0$, which completes the proof.
\end{proof}

\section{Experiments}
In this section, we conduct experiments on both synthetic and MNIST data to verify and complement our theoretical findings. Experiments on larger networks and data sets will left for a future work.
\subsection{Synthetic Data}\label{experiments}
Let $\left\{\bm e_1,\bm e_2,\bm e_3,\bm e_4\right\}$ be orthonormal basis of $\mathbb{R}^4$, $\theta$ be an acute angle and $\bm v_1=\bm e_1$, $\bm v_2=\sin\theta\,\bm e_2+\cos\theta\,\bm e_3$, $\bm v_3=\bm e_3$, $\bm v_4=\bm e_4$. Now, we have two linearly independent subspaces of $\mathbb{R}^4$ namely $\mathcal{V}_1=\text{Span}\left(\left\{\bm v_1,\bm v_2\right\}\right)$ and $\mathcal{V}_2=\text{Span}\left(\left\{\bm v_3,\bm v_4\right\}\right)$. We can easily calculate that the angle between $\mathcal{V}_1$ and $\mathcal{V}_2$ is $\theta$.
Next, with
$$S_r=\left\{\frac{j}{10}:j\in[20]-[9]\right\}, \;
S_\varphi=\left\{\frac{j\pi}{40}:j\in[80]\right\},$$
we define
$$\hat{\mathcal{X}}_1=\left\{r\left(\cos\varphi\,\bm v_1+\sin\varphi\,\bm v_2\right):r\in S_r,\varphi\in S_\varphi\right\}$$
and
$$\hat{\mathcal{X}}_2=\left\{r\left(\cos\varphi\,\bm v_3+\sin\varphi\,\bm v_4\right):r\in S_r,\varphi\in S_\varphi\right\}.$$
Let $\hat{\mathcal{D}}_i$ be uniform distributed on $\hat{\mathcal{X}}_i\times\{i\}$ and $\hat{\mathcal{D}}$ be a mixture of $\hat{\mathcal{D}}_1$ and $\hat{\mathcal{D}}_2$. Let $\hat{\mathcal{X}}=\hat{\mathcal{X}}_1\cup\hat{\mathcal{X}}_2$. 
The activation function $\sigma$ is 4-bit quantized ReLU:
$$\sigma(x)=\left\{\begin{array}{ccc}
     0& \text{if} &x<0,\\
     \text{ceil}(x)& \text{if} &0\leq x<15,\\
     15& \text{if} &x\geq15.\\
\end{array}\right.$$
For simplicity, we take $k=24$ and $v_{i,j}=\frac{1}{2}$ if $j-12(i-1)\in[12]$ for $i\in[2]$ and $j\in[24]$ and $0$ otherwise. Now, our neural network becomes
$$
f_i=\frac{(-1)^{i-1}}{2}\left[\sum_{j=1}^{12}\sigma(h_j)-\sum_{j=1}^{12}\sigma(h_{j+12})\right]
$$
where $h_j=\langle\bm w_j,\x\rangle$ and $\x\in\mathbb{R}^4$. 
%   Loss function
The population loss is given by
$$
l(\W)=\mathop{\mathbb{E}}_{\{\x,y\}\sim\hat{\mathcal{D}}}\left[l(\W;\{\x,y\})\right]
=\mathop{\mathbb{E}}_{\{\x,y\}\sim\hat{\mathcal{D}}}\left[\max\left\{1-f_i\right\}\right].
$$
%   Approximate coarse gradient
We choose the ReLU STE (i.e., $g(x) = \max\{0,x\}$) and use the coarse gradient
$$
\begin{aligned}
&\tilde{\nabla}_{\W}l(\W)=\mathop{\mathbb{E}}_{\{\x,y\}\sim\hat{\mathcal{D}}}\left[\tilde{\nabla}_{\W}l\left(\W,\{\x,y\}\right)\right]\\
=&\frac{1}{|\hat{\mathcal{X}}|}\left[\sum_{\x\in\hat{\mathcal{X}}_1}\tilde{\nabla}_{\W}l\left(\W;\{\x,1\}\right)+\sum_{\x\in\hat{\mathcal{X}}_2}\tilde{\nabla}_{\W}l\left(\W;\{\x,2\}\right)\right].
\end{aligned}
$$
Taking learning rate $\eta=1$, we have equation \ref{cgd} becomes
$$\W^{t+1}=\W^t-\tilde{\nabla}_{\W}l\left(\W^t\right).$$

%   stopping criterion
We find that the coarse gradient method converges to a global minimum with zero loss. 
%  explain figure angle_iter
As shown in box plots of Fig. \ref{angle_iters__angle_norms}, the {\it convergence still holds when the sub-spaces $\mathcal{V}_1$ and $\mathcal{V}_2$ form an acute angle}, and even when the data come from two levels of Gaussian noise perturbations of $\mathcal{V}_1$ and $\mathcal{V}_2$. The {\it convergence is faster} and with a smaller weight norm {\it when $\theta$ increases towards  $\frac{\pi}{2}$ or $\mathcal{V}_2$ are orthogonal to each other}. This observation clearly
supports the robustness of Theorem 1 
beyond the regime of orthogonal classes.

\begin{figure}[t]
    \centering
    \begin{tabular}{cc}
    \includegraphics[width=0.48\linewidth]{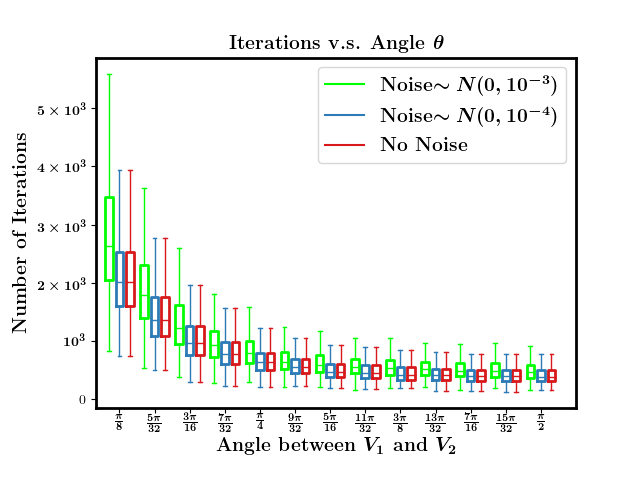}
    &
    \includegraphics[width=0.48\linewidth]{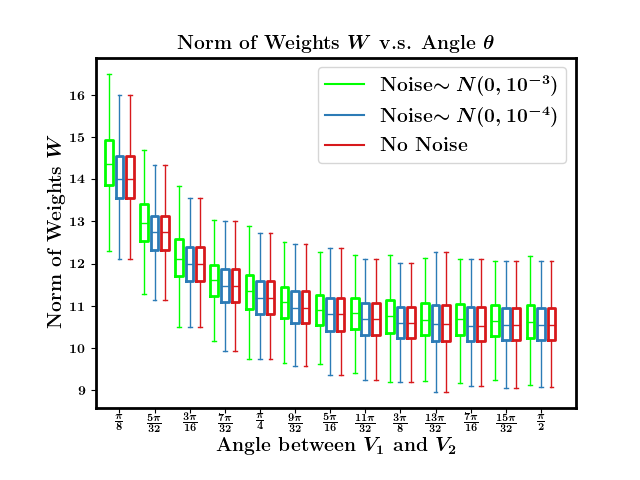}
    \end{tabular}
    \caption{\textbf{Left}: Iterations to convergence v.s. $\theta$, \textbf{Right}: Norm of weights v.s. $\theta$.}
    \label{angle_iters__angle_norms}
\end{figure}

\subsection{MNIST Experiments}

\begin{figure}[t]
    \centering
    \begin{tabular}{cc}
    \includegraphics[width=0.48\linewidth]{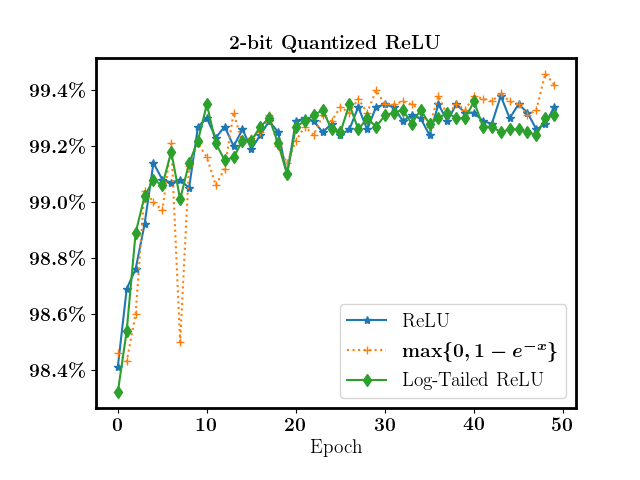} &
    \includegraphics[width=0.48\linewidth]{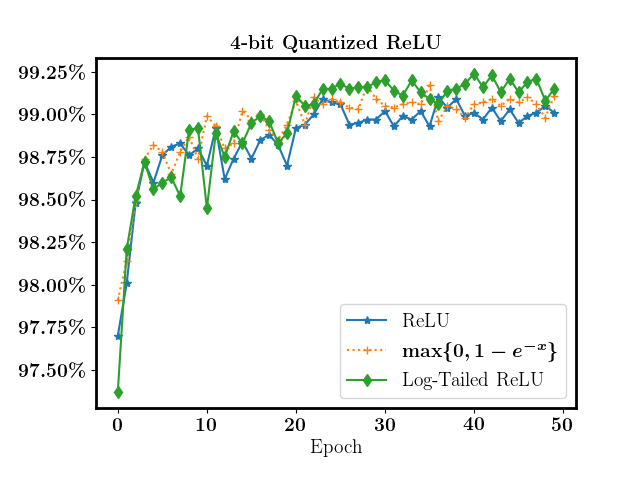}
    \end{tabular}
    \caption{Validation accuracies in training LeNet-5 with quantized (2-bit and 4-bit) ReLU activation.}
    \label{mnist_acc}
\end{figure}

%\begin{figure}[t]
%    \centering
%    \includegraphics[width=0.87\linewidth]{testacc.png}
%    \vspace{-1cm}
%    \caption{Validation accuracies in training LeNet-5 with quantized (2-bit and 4-bit) ReLU activation}
%    \label{mnist_acc}
%\end{figure}
% Batch size = 64
% Learning rate = 0.1
% Learning rate decay = 0.1
% Decay for every 20 epoch
% Momentum = 0.9
\begin{figure}
    \centering
    \includegraphics[width=0.8\linewidth]{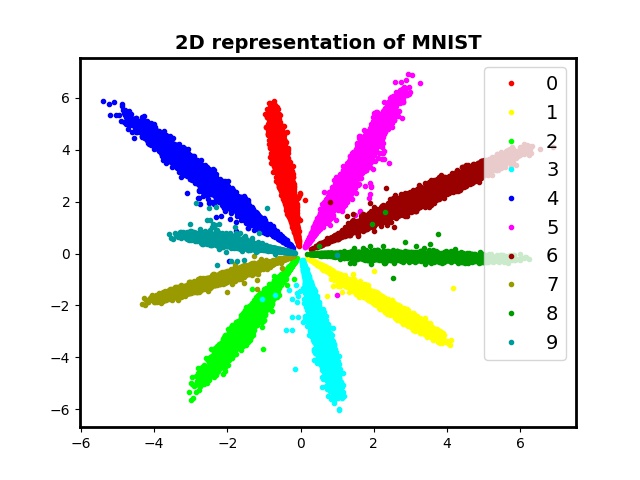}
    \caption{2D projections of MNIST features from a trained convolutional neural network \cite{cosface} with quantized activation function. The 10 classes are color coded, the feature points cluster near linearly independent subspaces.}
    \label{subspace}
\end{figure}

Our theory works for a board range of STEs, while their empirical performances on deeper networks may differ. In this subsection, we compare the performances of the three type of STEs in Fig. 2.

As in \cite{halfwave_17}, we resort to a modified batch normalization layer \cite{bnorm_15} and add it before each activation layer. As such, the inputs to quantized activation layers always follow unit Gaussian distribution. Then
the scaling factor $\tau$ applied to the output of quantized activation layers can be pre-computed via $k$-means approach  %\cite{yin2018blended} 
and get fixed during the whole training process.  The optimizer we use to train quantized LeNet-5 is the (stochastic) coarse gradient method with momentum = 0.9. The batch size is $64$, and learning rate is initialized to be $0.1$ and then decays by a factor of 10 after every $20$ epochs. The three backward pass substitutions $g$ for the straight through estimator are (1) ReLU $g(x) = \max\{x,0\}$, (2) reverse exponential $g(x)=\max\{0,q_b(1- e^{-x/q_b})\}$ (3) log-tailed ReLU. The validation accuracy for each epoch is shown in Fig. \ref{mnist_acc}. The validation accuracies at bit-widths 2 and 4 are listed in Table. \ref{mnist}. Our results show that these STEs all perform very well and give satisfactory accuracy. Specifically, reverse exponental and log-tailed STEs are comparable, both of which are slightly better than ReLU STE.
In Fig. \ref{subspace}, we show 2D projections of MNIST features at the end of 100 epoch training of a 7 layer convolutional neural network \cite{cosface} with quantized activation.  The features are extracted from input to the last fully connected layer. The data points cluster 
near linearly independent subspaces. Together with subsection 8.1, we have numerical evidence that the linearly independent subspace data structure (working as an extension of subspace orthogonality) occurs for high level features in a deep network for a nearly perfect  classification, rendering support to the realism of our theoretical study. Enlarging angles between linear subspaces can improve classification accuracy, see \cite{LFT_2018} for such an effort on MNIST and CIFAR-10 data sets via linear feature transform.

\begin{table}[ht]
\caption{Validation Accuracy (\%) on MNIST with LeNet5.}
\label{mnist}
\centering
\begin{tabular}{ccc}
\toprule
$g(x)$ & bit-width ($b$) & valid. accuracy\\
\midrule
  &32& 99.45\\
  \midrule
\multirow{2}{*}{ReLU} & 2& 99.10\\
 & 4& 99.38\\
 \midrule
\multirow{2}{*}{reverse exp.} & 2 & 99.17\\
 & 4 & 99.46\\
 \midrule
\multirow{2}{*}{log-tailed ReLU} & 2 &99.24\\
 & 4 &99.36\\
\bottomrule
\end{tabular}
\end{table}
\subsection{CIFAR-10 Experiments}
In this experiment, we train VGG-11/ResNet-20 with 4-bit activation function on CIFAR-10 data set to numerically validate the boundedness assumption upon the $\ell_2$-norm of weight. The optimizer is momentum SGD with no weight decay.  We used initial learning rate $=0.1$, with a decay factor of $0.1$ at the $80$-th and $140$-th epoch.

we see from Fig. \ref{fig:cifar-norm} that the $\ell_2$ norm of weights is bounded during the training process. This figure also shows that the norm of weights is generally increasing in epochs which coincides with our theoretical finding shown in Lemma \ref{descent}. 
\begin{figure}
    \centering
    \includegraphics[width=0.5\linewidth]{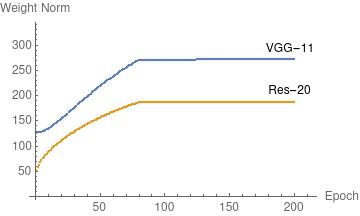}
    \caption{CIFAR-10 experiments for VGG-11 and ResNet-20: weight $\ell_2$-norm vs epoch.}
    \label{fig:cifar-norm}
\end{figure}
\section{Summary}
We studied a novel and important biased first-order oracle, called coarse gradient, 
%arising in the training problem 
in training quantized neural networks. The effectiveness of coarse gradient relies on the choice of STE used in backward pass only. We proved the convergence of coarse gradient methods for a class of STEs bearing certain monotonicity in non-linear classification using one-hidden-layer networks. In experiments on LeNet and MNIST data set, we considered three different proxy functions satisfying the monotonicity condition for backward pass: ReLU, reverse exponential function and log-tailed ReLU for training LeNet-5 with quantized activations. All of them exhibited good performance which verified our theoretical findings. 
In future work, we plan to expand theoretical understanding of coarse gradient descent for deep activation quantized networks.

\section{Acknowledgement}
This work was partially supported by NSF grants IIS-1632935, DMS-1854434, DMS-1924548, and DMS-1924935.
On behalf of all authors, the corresponding author states that there is no conflict of interest.

\bibliographystyle{siamplain}
\bibliography{reference}

\begin{thebibliography}{10}

\bibitem{bengio2013estimating}
{\sc Y.~Bengio, N.~L{\'e}onard, and A.~Courville}, {\em Estimating or
  propagating gradients through stochastic neurons for conditional
  computation}, arXiv preprint arXiv:1308.3432,  (2013).

\bibitem{halfwave_17}
{\sc Z.~Cai, X.~He, J.~Sun, and N.~Vasconcelos}, {\em Deep learning with low
  precision by half-wave gaussian quantization}, in IEEE Conference on Computer
  Vision and Pattern Recognition, 2017.

\bibitem{pact}
{\sc J.~Choi, Z.~Wang, S.~Venkataramani, P.~I.-J. Chuang, V.~Srinivasan, and
  K.~Gopalakrishnan}, {\em Pact: Parameterized clipping activation for
  quantized neural networks}, arXiv preprint arXiv:1805.06085,  (2018).

\bibitem{collobert2008unified}
{\sc R.~Collobert and J.~Weston}, {\em A unified architecture for natural
  language processing: Deep neural networks with multitask learning}, in
  International Conference on Machine Learning, ACM, 2008, pp.~160--167.

\bibitem{courbariaux2015binaryconnect}
{\sc M.~Courbariaux, Y.~Bengio, and J.-P. David}, {\em Binaryconnect: Training
  deep neural networks with binary weights during propagations}, in Advances in
  Neural Information Processing Systems, 2015, pp.~3123--3131.

\bibitem{ding2018universal}
{\sc Y.~Ding, J.~Liu, J.~Xiong, and Y.~Shi}, {\em On the universal
  approximability and complexity bounds of quantized relu neural networks},
  arXiv preprint arXiv:1802.03646,  (2018).

\bibitem{freund1999large}
{\sc Y.~Freund and R.~E. Schapire}, {\em Large margin classification using the
  perceptron algorithm}, Machine learning, 37 (1999), pp.~277--296.

\bibitem{he2018relu}
{\sc J.~He, L.~Li, J.~Xu, and C.~Zheng}, {\em Re{LU} deep neural networks and
  linear finite elements}, Journal of Computational Mathematics, 38 (2020),
  pp.~502--527.

\bibitem{hinton2012neural}
{\sc G.~Hinton}, {\em Neural networks for machine learning, coursera},
  Coursera, video lectures,  (2012).

\bibitem{hou2018loss}
{\sc L.~Hou and J.~T. Kwok}, {\em Loss-aware weight quantization of deep
  networks}, in International Conference on Learning Representations, 2018.

\bibitem{bnn_16}
{\sc I.~Hubara, M.~Courbariaux, D.~Soudry, R.~El-Yaniv, and Y.~Bengio}, {\em
  Binarized neural networks}, in Advances in Neural Information Processing
  Systems, 2016.

\bibitem{Hubara2017QuantizedNN}
{\sc I.~Hubara, M.~Courbariaux, D.~Soudry, R.~El-Yaniv, and Y.~Bengio}, {\em
  Quantized neural networks: Training neural networks with low precision
  weights and activations}, Journal of Machine Learning Research, 18 (2018),
  pp.~1--30.

\bibitem{bnorm_15}
{\sc S.~Ioffe and C.~Szegedy}, {\em Batch normalization: Accelerating deep
  network training by reducing internal covariate shift}, in International
  Conference on Machine Learning, 2015.

\bibitem{imagnet_12}
{\sc A.~Krizhevsky, I.~Sutskever, and G.~E. Hinton}, {\em Imagenet
  classification with deep convolutional neural networks}, in Advances in
  Neural Information Processing Systems, 2012, pp.~1097--1105.

\bibitem{twn_16}
{\sc F.~Li, B.~Zhang, and B.~Liu}, {\em Ternary weight networks}, arXiv
  preprint arXiv:1605.04711,  (2016).

\bibitem{li2017training}
{\sc H.~Li, S.~De, Z.~Xu, C.~Studer, H.~Samet, and T.~Goldstein}, {\em Training
  quantized nets: A deeper understanding}, in Advances in Neural Information
  Processing Systems, 2017, pp.~5811--5821.

\bibitem{louizos2018relaxed}
{\sc C.~Louizos, M.~Reisser, T.~Blankevoort, E.~Gavves, and M.~Welling}, {\em
  Relaxed quantization for discretized neural networks}, in International
  Conference on Learning Representations, 2019.

\bibitem{mnih2015human}
{\sc V.~Mnih, K.~Kavukcuoglu, D.~Silver, A.~A. Rusu, J.~Veness, M.~G.
  Bellemare, A.~Graves, M.~Riedmiller, A.~K. Fidjeland, G.~Ostrovski, et~al.},
  {\em Human-level control through deep reinforcement learning}, Nature, 518
  (2015), p.~529.

\bibitem{faster_rcnn}
{\sc S.~Ren, K.~He, R.~Girshick, and J.~Sun}, {\em Faster {R-CNN}: Towards
  real-time object detection with region proposal networks}, in Advances in
  Neural Information Processing systems, 2015, pp.~91--99.

\bibitem{rosenblatt1957perceptron}
{\sc F.~Rosenblatt}, {\em The perceptron, a perceiving and recognizing
  automaton Project Para}, Cornell Aeronautical Laboratory, 1957.

\bibitem{rosenblatt1962principles}
{\sc F.~Rosenblatt}, {\em Principles of neurodynamics}, Spartan Book, 1962.

\bibitem{shen2020deep}
{\sc Z.~Shen, H.~Yang, and S.~Zhang}, {\em Deep network approximation with
  discrepancy being reciprocal of width to power of depth}, arXiv preprint
  arXiv:2006.12231,  (2020).

\bibitem{silver2016mastering}
{\sc D.~Silver, A.~Huang, C.~J. Maddison, A.~Guez, L.~Sifre, G.~Van
  Den~Driessche, J.~Schrittwieser, I.~Antonoglou, V.~Panneershelvam,
  M.~Lanctot, et~al.}, {\em Mastering the game of go with deep neural networks
  and tree search}, Nature, 529 (2016), p.~484.

\bibitem{cosface}
{\sc H.~Wang, Y.~Wang, Z.~Zhou, X.~Ji, Z.~Li, D.~Gong, J.~Zhou, and W.~Liu},
  {\em Cosface: Large margin cosine loss for deep face recognition}, in IEEE
  Conference on Computer Vision and Pattern Recognition, 2008.

\bibitem{widrow199030}
{\sc B.~Widrow and M.~A. Lehr}, {\em 30 years of adaptive neural networks:
  perceptron, madaline, and backpropagation}, Proceedings of the IEEE, 78
  (1990), pp.~1415--1442.

\bibitem{yin2018understanding}
{\sc P.~Yin, J.~Lyu, S.~Zhang, S.~J. Osher, Y.~Qi, and J.~Xin}, {\em
  Understanding straight-through estimator in training activation quantized
  neural nets}, in International Conference on Learning Representations, 2019.

\bibitem{LFT_2018}
{\sc P.~Yin, J.~Xin, and Y.~Qi}, {\em Linear feature transform and enhancement
  of classification on deep neural network}, Journal of Scientific Computing,
  76 (2018), pp.~1396--1406.

\bibitem{yin2018binaryrelax}
{\sc P.~Yin, S.~Zhang, J.~Lyu, S.~Osher, Y.~Qi, and J.~Xin}, {\em Binaryrelax:
  A relaxation approach for training deep neural networks with quantized
  weights}, SIAM Journal on Imaging Sciences, 11 (2018), pp.~2205--2223.

\bibitem{yin2018blended}
{\sc P.~Yin, S.~Zhang, J.~Lyu, S.~Osher, Y.~Qi, and J.~Xin}, {\em Blended
  coarse gradient descent for full quantization of deep neural networks},
  Research in the Mathematical Sciences, 6 (2019).

\bibitem{yin2016quantization}
{\sc P.~Yin, S.~Zhang, Y.~Qi, and J.~Xin}, {\em Quantization and training of
  low bit-width convolutional neural networks for object detection}, Journal of
  Computational Mathematics, 37 (2019), pp.~349--359.

\bibitem{inq_17}
{\sc A.~Zhou, A.~Yao, Y.~Guo, L.~Xu, and Y.~Chen}, {\em Incremental network
  quantization: Towards lossless {CNN}s with low-precision weights}, in
  International Conference on Learning Representations, 2017.

\bibitem{dorefa_16}
{\sc S.~Zhou, Y.~Wu, Z.~Ni, X.~Zhou, H.~Wen, and Y.~Zou}, {\em Do{r}e{f}a-net:
  Training low bitwidth convolutional neural networks with low bitwidth
  gradients}, arXiv preprint arXiv:1606.06160,  (2016).

\bibitem{ttq_16}
{\sc C.~Zhu, S.~Han, H.~Mao, and W.~J. Dally}, {\em Trained ternary
  quantization}, in International Conference on Learning Representations, 2017.

\end{thebibliography}
\end{document}